\documentclass[11pt]{article}

\usepackage{amsthm,amssymb,amsfonts,amsmath}
\usepackage{amscd} 
\usepackage{mathrsfs}
\usepackage{graphicx}
\usepackage{color}
\usepackage{hyperref}
\usepackage{bm}

\usepackage{biblatex} 
\newtheorem{dfn}{Definition}
\newtheorem{thm}{Theorem}
\newtheorem{lem}{Lemma}
\newtheorem{cor}{Corollary}

\begin{document}

\title{Probabilistic Analysis of \textsc{rrt} Trees}
\author{Konrad Anand and Luc Devroye}
\date{School of Computer Science,\\
McGill University,\\
Montreal, Canada}
\maketitle

\begin{abstract}
This thesis presents analysis of the properties and run-time of the Rapidly-exploring Random Tree (\textsc{rrt}) algorithm. It is shown that the time for the \textsc{rrt} with stepsize $\epsilon$ to grow close to every point in the $d$-dimensional unit cube is $\Theta\left(\frac1{\epsilon^d} \log \left(\frac1\epsilon\right)\right)$. Also, the time it takes for the tree to reach a region of positive probability is $O\left(\epsilon^{-\frac32}\right)$. Finally, a relationship is shown to the Nearest Neighbour Tree (\textsc{nnt}). This relationship shows that the total Euclidean path length after $n$ steps is $O(\sqrt n)$ and the expected height of the tree is bounded above by $(e + o(1)) \log n$.
\end{abstract}
\newpage
\tableofcontents
\newpage

\begin{figure*}
\centering
\includegraphics[width=0.45\textwidth]{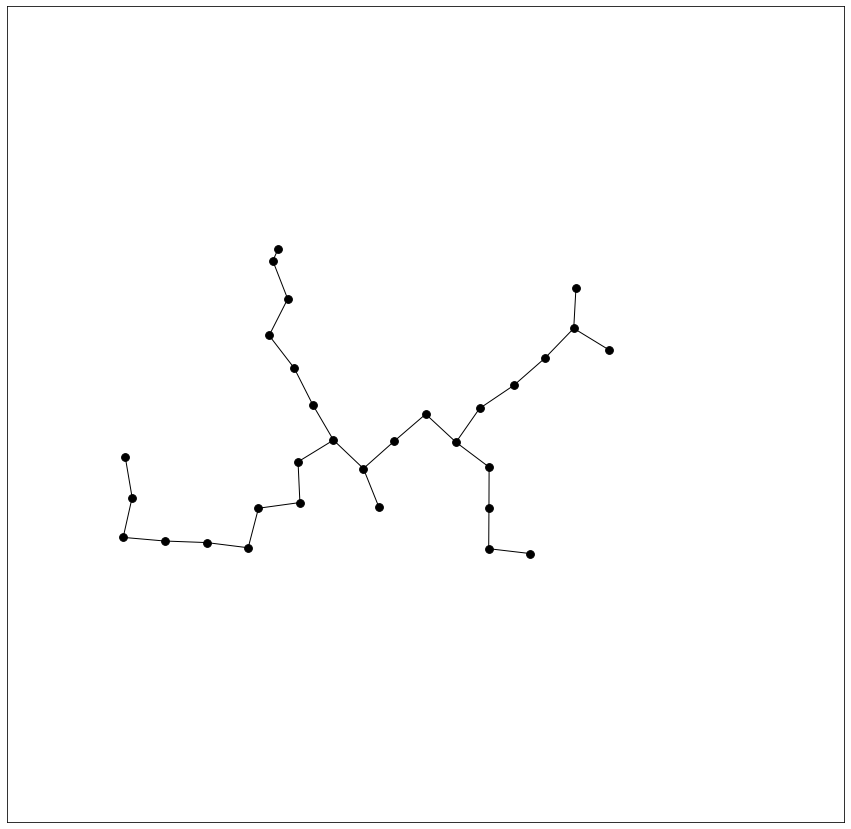}
\includegraphics[width=0.45\textwidth]{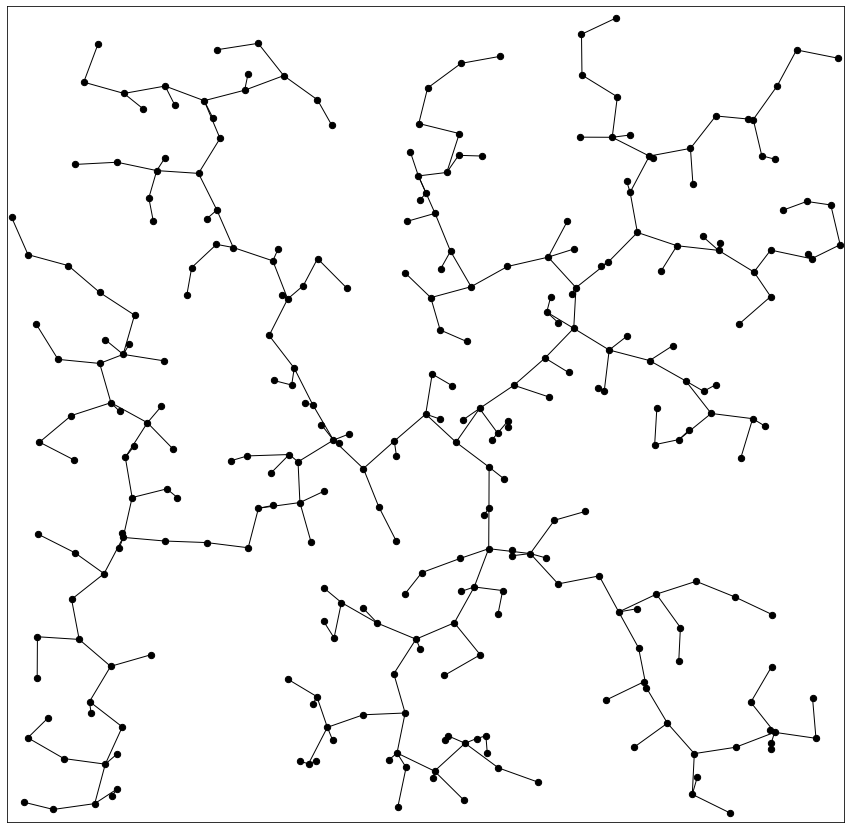}
\caption{The growth of the \textsc{rrt} at times $\Theta\left(\frac1\epsilon\right)$ left and $\Theta\left(\frac1{\epsilon^2}\right)$ right. Above the \textsc{rrt} is expanding outwards quickly. Below the \textsc{rrt} has begun to fill holes left by its expansion but is far from covering the space.}
\end{figure*}

\section{Introduction}

In 1998, LaValle \cite{Lavalle98rapidly-exploringrandom} introduced a ground-breaking approach to path planning: the Rapidly-exploring Random Tree (\textsc{rrt}). At the time the area of path planning had already been taken over by randomized algorithms due to strong evidence of a limit on the potential speed of deterministic path  \cite{kuffner2000rrt, reif1979complexity, reif1985complexity}. The \textsc{rrt} was a novel randomized approach that has quickly come to dominate modern path planning. Previously the most widely used algorithm had been the probabilistic road map \cite{kavraki1996probabilistic, dalibard2009control}, an algorithm which randomly selects points in the space and connects these points to the graph. The \textsc{rrt} differed, in that rather than adding these randomly selected points, it built a tree with small edge length based on these points.

The \textsc{rrt}'s ubiquity is in large part due to a number of variants to the original algorithm, such as the \textsc{rrt}-Connect \cite{kuffner2000rrt, lavalle2001randomized} and the \textsc{rrt}*  \cite{karaman2011sampling, islam2012rrt}. In the former, we grow an \textsc{rrt} out of both the initial and final point of a desired path, while trying to connect them. In the latter, to optimize the spanning ratio we grow our tree based on a choice of minimal path length as well as the locations of the randomly selected points. These modifications and others are missing analytical grounding to understand their performance.

There have been prior analyses of the \textsc{rrt} but none in quite this area. Soon after the creation of the \textsc{rrt}, Kuffner and LaValle \cite{kuffner2000rrt} proved that the algorithm was probabilistically complete---i.e., the algorithm succeeds with high probability---and that the distribution of the $n$-th vertex in the tree converges to the sampling distribution. In their analysis however, they commented that they had no result on the speed of this convergence. Further analysis has been done by Arnold et al. \cite{arnold2013convex} on the speed at which the \textsc{rrt} spreads out by studying an idealized version of the algorithm on an infinite disc. Work has also been done by Karaman and Frazzoli \cite{karaman2011sampling} on the complexity of the update step in the algorithm and the optimality of the path, while Dalibard and Laumond \cite{dalibard2009control} have studied the effects of narrow passages on the performance of the algorithm. However there remains a gap in the literature on determining a run time until success of the \textsc{rrt}. This paper aims to fill this gap.

\subsection{Organization}

The paper is organized as follows: first we set out definitions and a framework for our approach. We then split our analysis into three phases of growth of the \textsc{rrt}.

The first phase is the phase where the \textsc{rrt} spreads out quickly. Most of the space is still far away from the tree, so most steps give a large amount of progress. This is the phase most similar to the problem considered by Arnold et al.\cite{arnold2013convex}.

The second phase is the phase where the \textsc{rrt} grows close to every point in the space. Kuffner and LaValle \cite{kuffner2000rrt} proved that this occurs and we will add information on the speed.

Finally, the third phase covers the growth of the tree once the \textsc{rrt} has grown close to every point in the graph. We find a relationship here to the nearest neighbour tree which allows easy study of the tree and many results on its behaviour.

\section{Definitions}
\subsection{Initial Framework}

\begin{figure}[ht]
\centering
\includegraphics[width=0.5\textwidth]{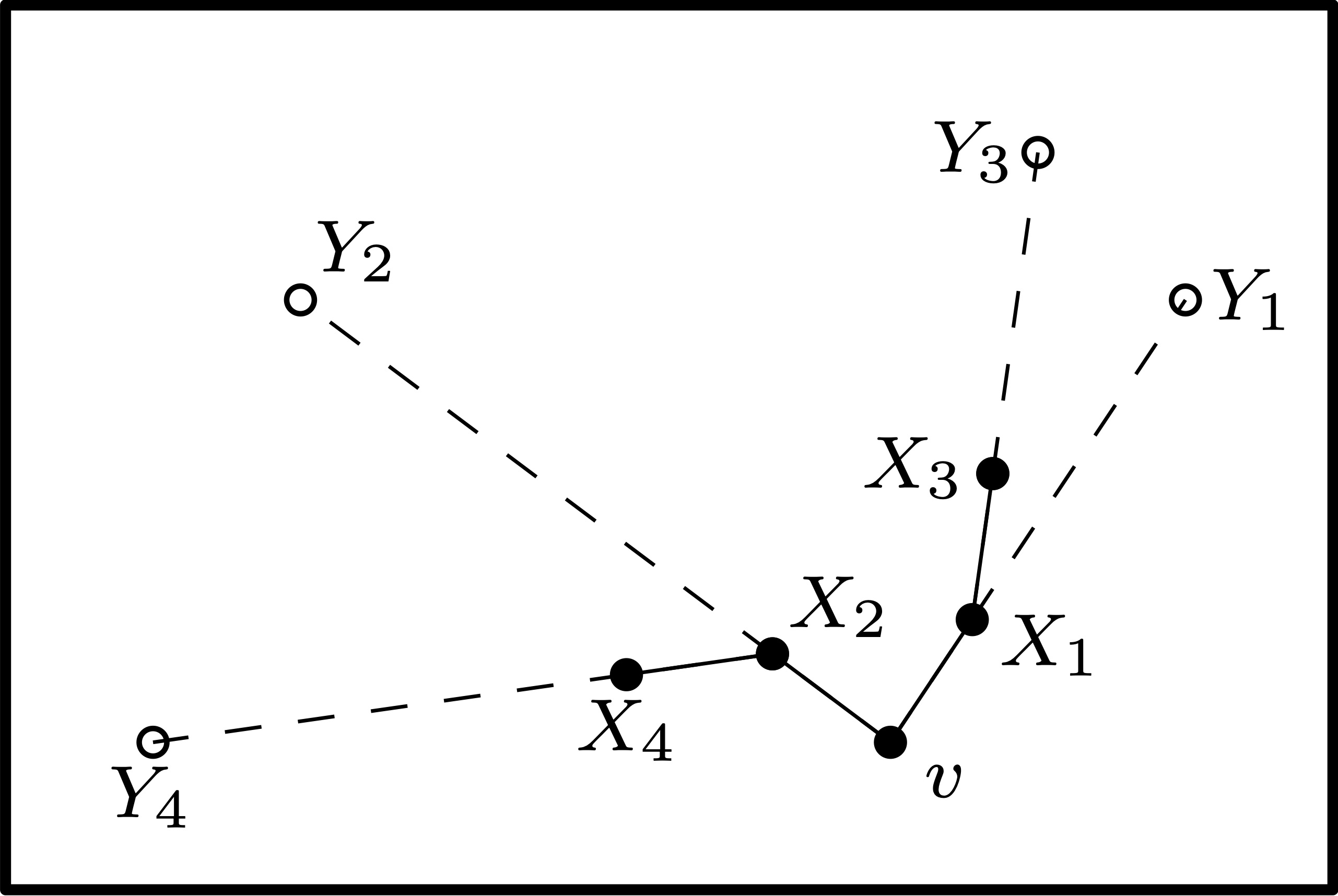}
\caption{The start of an \textsc{rrt}. The $X_i$'s are the vertices and the $Y_i$'s are their determining uniform random variables.}
\end{figure}

We begin with basic definitions of the \textsc{rrt} on compact, convex metric spaces. In lieu of the vocabulary used in general path planning problems, we will use the vocabulary of metric spaces and probability.

$\newline$
\begin{dfn}
An \textsc{rrt} is a sequence of vertex and edge random variable pairs
\begin{align*}
    T_\epsilon(M, v) := ((X_i,E_i))_{i \in \mathbb N}
\end{align*}
defined on a compact metric space $M$ with distance function $d$, a step-size parameter $\epsilon > 0$, $E_0 = \emptyset$, and an initial vertex $v = X_0 \in M$. Each $(X_i,E_i), i \geq 1$ is generated in the following way: take a sequence of uniform $M$-valued random variables $(Y_i)_{i=1}^\infty$ and let
\begin{align*}
    \gamma_i = \textnormal{argmin}_{0 \leq j \leq i-1} d(Y_i,X_j).
\end{align*}
Define
\begin{align*}
    X_i &:=
    \begin{cases}
      Y_i &: \text{if}\ \min_{0 \leq j \leq i-1} d(x_i,v_j) \leq     \epsilon, \\
      X_{\gamma_i} + \epsilon\frac{Y_i - X_{\gamma_i}}{d(Y_i,X_{\gamma_i})} &: \textnormal{otherwise}, 
    \end{cases}
\end{align*}
and
\begin{align*}
    E_i &:= (X_{\gamma_i},X_i).
\end{align*}
An \textsc{rrt} at time $n \in \mathbb N$ is the subset $T_\epsilon^n(M,v) \subset T_\epsilon(M,v)$ given by
\begin{align*}
    T_\epsilon^n(M, v) := ((X_i,E_i))_{1 \leq i \leq n}.
\end{align*}
\end{dfn}

In plain terms, we start with an initial vertex $v$ and then iteratively we select a point $Y_i$ uniformly in the space, find $X_{\gamma_i}$---the closest vertex in the tree at time $i-1$---and then add a vertex-edge pair in the direction of $Y_i$ with edge length less than or equal to $\epsilon$. In this definition we do not check for obstacles blocking the edge as we are assuming the space to be convex.

To remove clunkiness in notation, we often leave $M$ and $v$ implicit and just write $T_\epsilon^n$. Typically, $M$ will be the $d$-dimensional unit cube $[0,1]^d$. Also, the initial vertex $v$ is often irrelevant to our analysis.

\subsection{Covering Metric Spaces}

One measure of success for the \textsc{rrt} is whether or not the tree has grown close to every point in the space. A natural definition for ``close'' for an \textsc{rrt} $T_\epsilon^n$ is being within $\epsilon$ of the tree.

\begin{dfn}
Consider a compact metric space $M$ and a graph $G = (V,E) \subset M$ where $V$ is the set of vertices and $E$ is the set of edges. The $\epsilon$-cover of $G$ is
\begin{align*}
    C_\epsilon(G) := \bigcup_{v \in V} B(v,\epsilon)
\end{align*}
where $B(x,\epsilon)$ is the ball
\begin{align*}
    B(x,\epsilon) = \left\{y \in M : d(x,y) \leq \epsilon \right\}.
\end{align*}
In the case of an \textsc{rrt} $T_\epsilon^n = ((X_i,E_i))_{1 \leq i \leq n}$, the cover is
\begin{align*}
    C(T_\epsilon^n) := \bigcup_{0 \leq i \leq n} B(X_i,\epsilon).
\end{align*}
\end{dfn}

\begin{dfn}
The covering time of a metric space $M$ by an \textsc{rrt}  $T_\epsilon^n$ is the $\mathbb N$-valued random variable
\begin{align*}
    \tau(T_\epsilon) = \min \{n \in \mathbb N : C(T_\epsilon^n) \supset M\}.
\end{align*}
\end{dfn}

In particular, we will analyze the expected covering time, $\mathbb E \left[\tau(T_\epsilon)\right]$. Unless explicitly noted, we assume that $M = [0,1]^2$.s

\section{Phase 1---Initial Growth}

\begin{figure}
\centering
\includegraphics[width=0.7\textwidth]{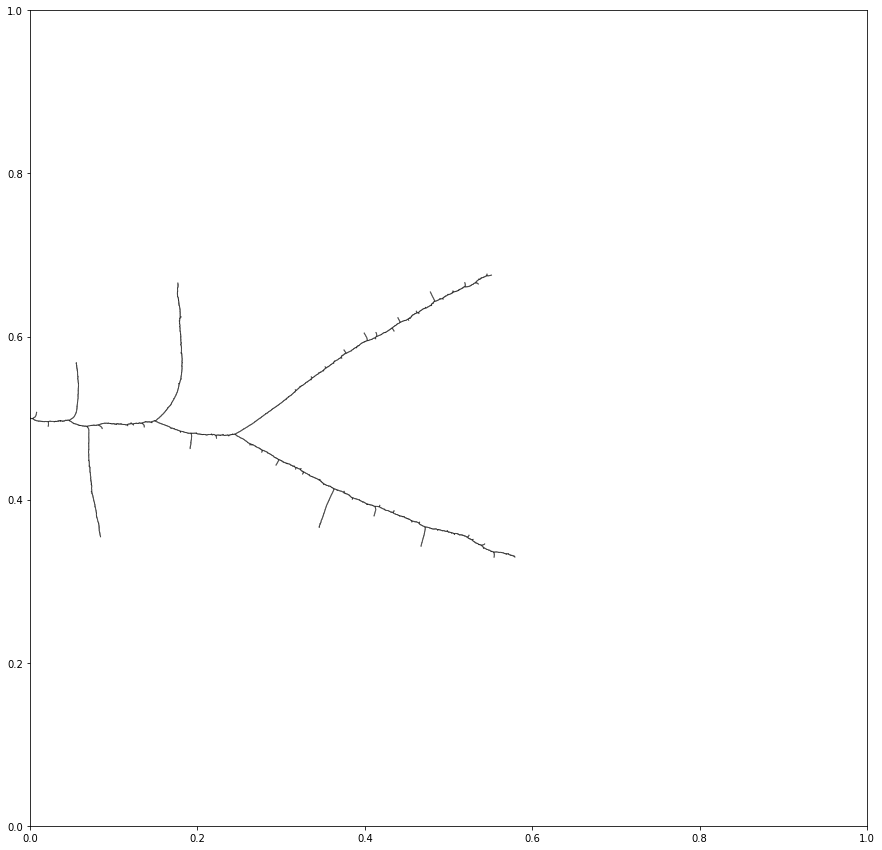}
\caption{The first phase of an \textsc{rrt}.}
\end{figure}

We begin by studying how quickly the \textsc{rrt} spreads. Arnold et al. \cite{arnold2013convex} consider a modification to the \textsc{rrt} where the space is the plane $\mathbb R^2$, the step size is 1, and points are drawn at infinity by picking a ray from the origin with a uniformly random direction. They prove in this model that the perimeter of the convex hull of tree grows linearly in time which implies that the maximum distance from the tree to the origin also grows linearly in time. Here, we consider a related question: how long does it take an \textsc{rrt} to reach a convex region of positive probability?

While a number of such regions can be chosen, for ease of calculation we consider the case that we have an \textsc{rrt} on $[0,1]^2$ and the region we want to reach is the right half of the square, $\left[\frac12,1\right] \times [0,1]$ when $X_0 = (0,0).$ Similar arguments extend to the general convex region.

\begin{thm}
For an \textsc{rrt} $T_\epsilon$ on $[0,1]^2$, the expected time to reach $\left[\frac12,1\right] \times [0,1]$ is $O\left(\epsilon^{-\frac32}\right)$ as $\epsilon \to 0$.
\end{thm}

\begin{proof}
We will follow the vertex with greatest $x$-value through the whole process. Consider the Voronoi region of this point. There are two cases: either there are multiple vertices with the same $x$-value or there is a single vertex with greatest $x$-value. If there are multiple vertices with the same $x$-value then then we have a greater chance of useful progression so we focus on the case of a unique vertex.

Suppose we have a unique vertex with greatest $x$-value, $v$. Consider any other vertex in the tree, $w$. This vertex only impacts the Voronoi region of $v$ by the perpendicular bisector of the line between $v$ and $w.$ Following this, we look at where $v$'s Voronoi region intersects the right edge of the box, $\{1\} \times [0,1]$ and see that it does so at the edge of the Voronoi region of some other vertex. Also, we see that the Voronoi region contains a triangle from the vertex $v$ to the intersections of its Voronoi region with $\{1\} \times [0,1]$.

\begin{figure}
\centering
\includegraphics[width=0.45\textwidth]{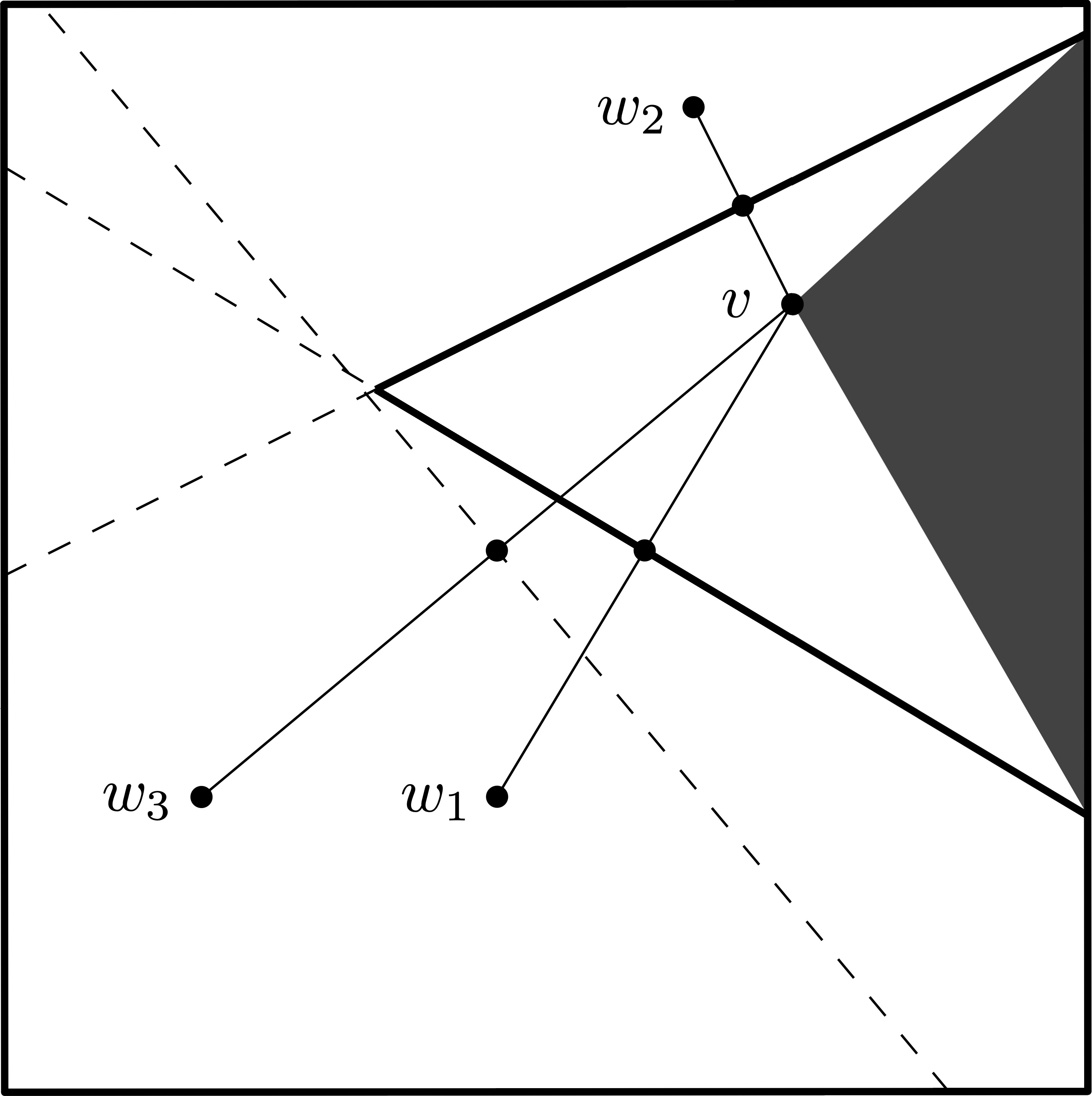}
\caption{The Voronoi region of $v$ made of its pairwise Voronoi regions.}
\end{figure}

Our goal is to find a probability of progress which we will do by looking at this triangle---more specifically by considering triangles and considering when the Voronoi region would contain this triangle. In the worst case scenario, $v$ is on the upper or lower edge of $[0,1]^2$, so without loss of generality, we take $v$ to be on the upper edge.

Consider the triangle $\Lambda$ with points $v, (1,1),$ and $(1,1-\delta)$ for some $\delta > 0$. The first question to ask is when will the Voronoi region of $v$ not contain $\Lambda$? We find the region where any vertex other that $v$ inside will have a Voronoi region intersecting $\Lambda$. Observe that for any point $w$ on the boundary of this region, the perpendicular bisector of the line between $v$ and $w$ must pass through $(1,1-\delta)$. Writing
\begin{align*}
    w = v - (x,y) = (1-z - x,1 - y),
\end{align*}
consider the similar triangles
\begin{align*}
    \left\{v, v - \left(\frac x2, \frac y2\right), v - \left(0, \frac y2\right)\right\}
\end{align*}
and
\begin{align*}
    \left\{(1, 1 - \delta), \left(1, 1 - \frac y2\right), v - \left(\frac x2, \frac y2\right) \right\}.
\end{align*}
We have the equality
\begin{align*}
    \frac xy = \frac{\delta - \frac y2}{z + \frac x2},
\end{align*}
which implies that the boundary points $w = (x_0, y_0)$ all lie on the circle $\mathcal C$ whose points satisfy
\begin{align*}
    &\left(x + z\right)^2 + \left(y - \delta\right)^2 \\
    =\ &\left(x_0-1\right)^2 + (y_0 - 1 + \delta)^2 \\
    =\ &\delta^2 + z^2,
\end{align*}
which is the circle centered at $(1,1-\delta)$ with radius $\sqrt{\delta^2 + z^2}$.

\begin{figure}
\centering
\includegraphics[width=0.7\textwidth]{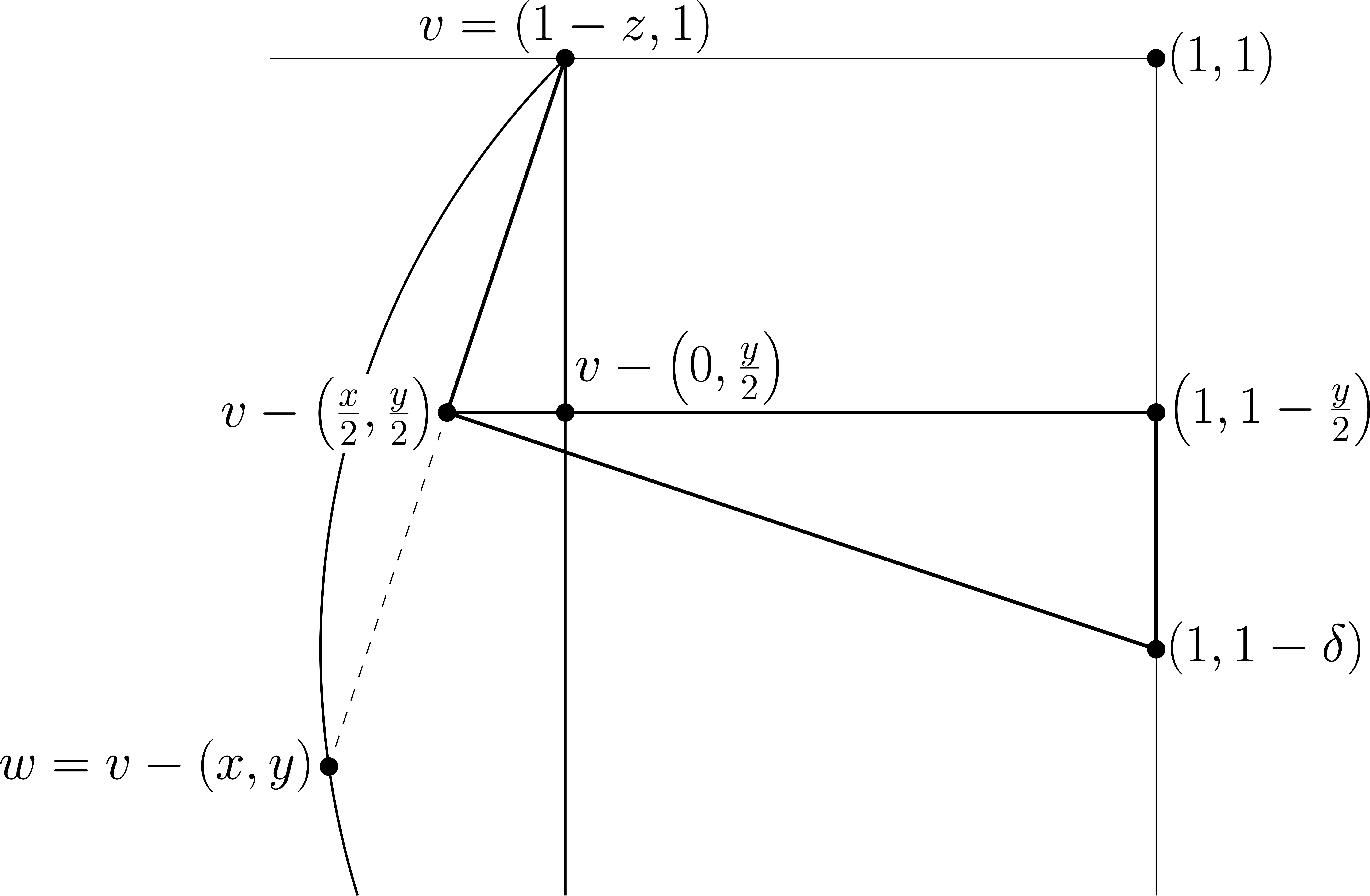}
\caption{The similar triangles $\left\{v, v - \left(\frac x2, \frac y2\right), v - \left(0, \frac y2\right)\right\}$ and $\left\{(1, 1 - \delta), \left(1, 1 - \frac y2\right), v - \left(\frac x2, \frac y2\right) \right\}$.}
\end{figure}

The area that could interfere with this circle is then a cap with height
\begin{align*}
    h(\delta,z) = \sqrt{\delta^2 + z^2} - z.
\end{align*}
Note that $\frac12 \leq z \leq 1$, and that $h(\delta)$ is decreasing in $z$, so we may upper bound $h(\delta,z)$ with
\begin{align*}
    h(\delta,z) \leq \sqrt{\delta^2 + \frac14} := h(\delta).
\end{align*}

Now we pick $\delta$ in such a way that any time we select a point in $\Lambda$ we make $\Theta(\epsilon)$ progress. Set $\delta = \frac12 \sqrt\epsilon$ and take $\epsilon < \frac1{16}$. Suppose our $i$-th uniform vertex $Y_i \in \Lambda$. The easy case is that $X_{\gamma_i} = v$ in which case we make progress of at least $\frac1{\sqrt2} \epsilon.$ The more subtle case is if $v \not= X_{\gamma_i}$. Then $X_{\gamma_i}$ is in the circle $\mathcal C$ and since $v$ is the most advanced point, $X_{\gamma_i}$ is in fact in the cap. $X_i$ is at least $\frac1{\sqrt2}\epsilon$ further than $X_{\gamma_i}$, but how much progress have we actually made? We must analyse the difference between the height of the cap and $Z_i$'s progress.

We will show that the height of the cap is less than $\frac12 \epsilon$ by a constant fraction of $\epsilon$. We derive $\frac1{\sqrt2}\epsilon - h\left(\frac12\sqrt\epsilon\right)$ to get
\begin{align*}
    &\frac{d}{d\epsilon} \left(\frac1{\sqrt2}\epsilon - \sqrt{\frac14 \epsilon + \frac14}\right) \frac1{\sqrt2} - \frac1{8 \sqrt{\frac14 \epsilon + \frac14}}.
\end{align*}
This is increasing in $\epsilon$ and at $0$ it is $\frac1{\sqrt2} - \frac14$, so it follows that we progress by at least $\left(\frac1{\sqrt2} - \frac12\right) \epsilon.$

\begin{figure}
\centering
\includegraphics[width=0.6\textwidth]{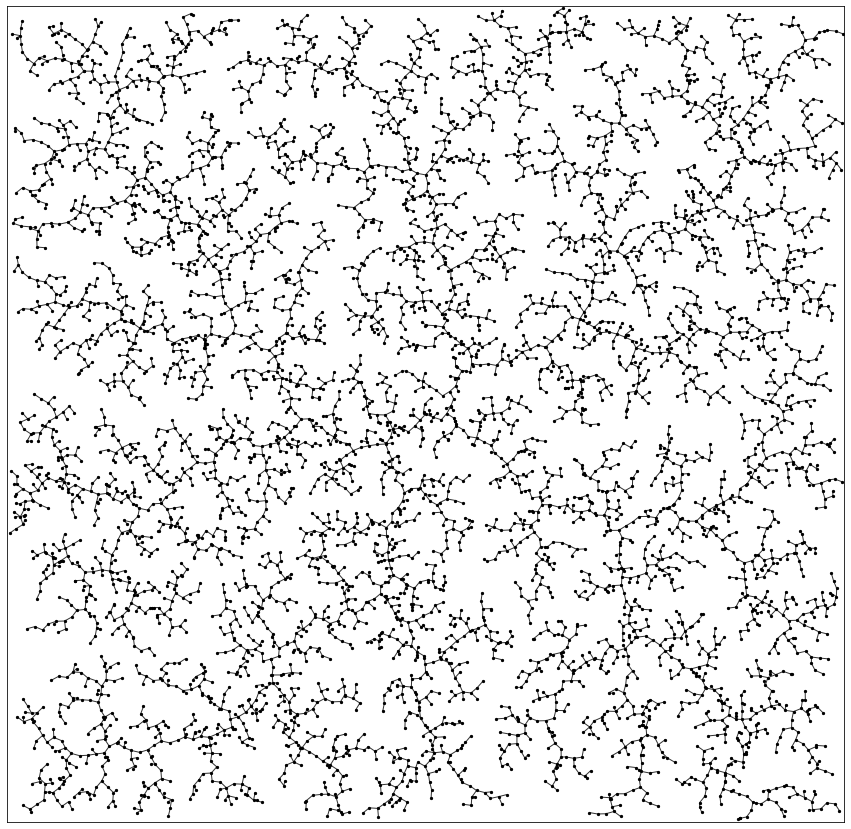}
\caption{A vertex $w$ inside the cap whose Voronoi region cuts into $\Lambda$.}
\end{figure}

This shows that it suffices to select a point in $\Lambda$ at least $\frac1{\left(\frac1{\sqrt2} - \frac14\right)\epsilon}$ times to reach the right half of the unit square. Meanwhile, the area of $\Lambda$ is at least $\frac14 \sqrt\epsilon$ so the expected time to select a point in $\Lambda$ is less than $\frac4{\sqrt\epsilon}$. It follows that the expected time to reach $\left[\frac12,1\right] \times [0,1]$ is less than
\begin{align*}
    \frac1{\left(\frac1{\sqrt2} - \frac12\right)\epsilon} \cdot \frac4{\sqrt\epsilon} = O\left(\epsilon^{-\frac32}\right).
\end{align*}
\end{proof}

It is worth remarking that if we repeat the analysis in $[0,1]^d$ and wish to reach $\left[\frac12,1\right] \times [0,1]^{d-1}$ the argument above holds with slightly different constants.

The question remains whether this upper bound is tight or not. One can imagine that in expectation either some finite number of branches shoot out rapidly towards the right half of $[0,1]^2$---in this case we would have in fact $\Theta(\epsilon^{-1})$---or that the tree will progress with tendrils approximately $\sqrt\epsilon$ apart from each other moving forward evenly---this being the $\Theta\left(\epsilon^{-\frac32}\right)$ case. The question remains open.

\section{Phase 2---Covering Times}

\begin{figure}[ht]
\centering
\includegraphics[width=0.7\textwidth]{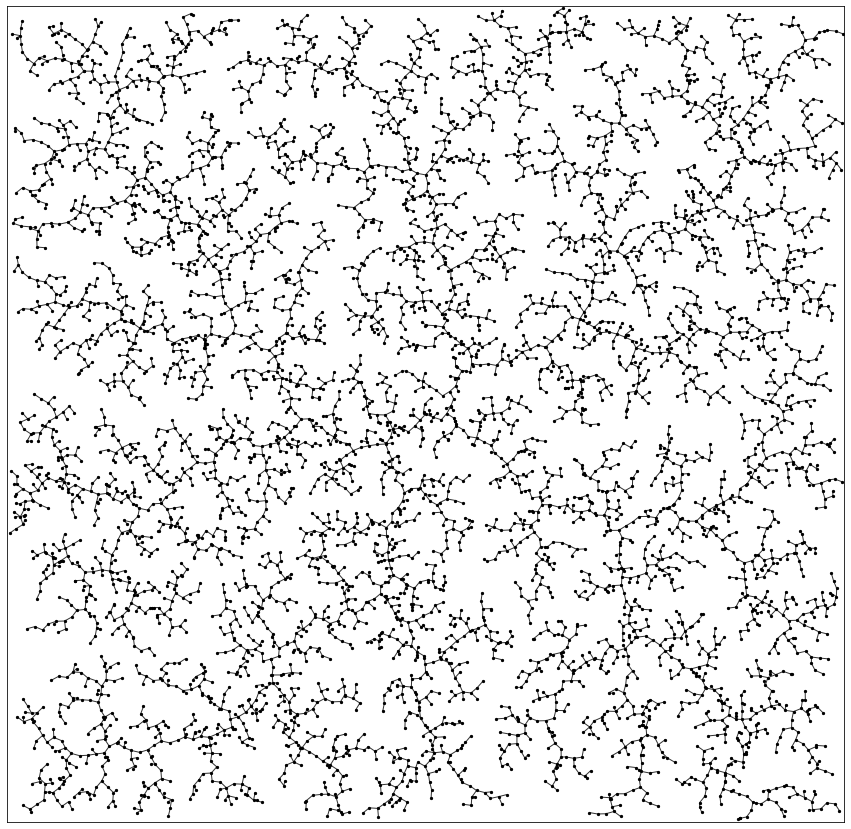}
\caption{An \textsc{rrt} shortly before covering. Some parts of the space are still uncovered, while others have begun nearest neighbour behaviour.}
\end{figure}

\subsection{Covering $[0,1]^d$}
We turn our attention to the covering time for $M = [0,1]^d$ in the Euclidean metric. This gives the expected run time of the algorithm to some measure of success. The following analysis does not depend on the initial vertex $v$ so we drop it. First we look at the expected covering time for an \textsc{rrt} on $[0,1]^d$.

\begin{thm}
The expected covering time of $[0,1]^d$ by an \textsc{rrt} $T_\epsilon$ is

\begin{align*}
    \mathbb E \left[\tau(T_\epsilon)\right] = \Theta\left(\frac1{\epsilon^d} \log \frac1\epsilon\right).
\end{align*}
\end{thm}

\begin{proof}
We first find a lower bound for $\mathbb E \left[\tau(T_\epsilon)\right]$. We consider a critical subset of the time steps: the time steps where new area is added to the cover. We work backwards from covering. Our goal is to lower bound the amount of time required for each critical step.

Define $\beta$ to be the inverse of the volume of the $d$-dimensional ball with radius $\epsilon$:
\begin{align*}
\beta := \frac{\Gamma\left(\frac d2 + 1\right)}{\epsilon^d \pi^\frac d2}, \quad \overline{\beta} := \left \lfloor \beta \right \rfloor.
\end{align*}
The maximum amount of space we cover at any critical step is $\frac1\beta$ so the number of critical steps is at least $\overline \beta$. As well, we will argue that the probability of achieving the $\ell$-th last critical step is bounded by $\ell \frac{2^d}\beta$.

Consider the final critical step of any \textsc{rrt}. The area added in this step is a possibly uncountable set of points contained in an $\epsilon$-ball. Letting $i$ be any time before the final critical step and after the penultimate critical step, we note that we may only have a critical step if $Y_i$ is within $\epsilon$ of an uncovered point. This implies that the set of points that could lead to a critical step is contained in a $2\epsilon$-ball.

This argument can be repeated inductively---at any time the set of points that could lead to a critical step is contained in a union of $2\epsilon$-balls. To be precise, the time taken for the $\ell$-th last critical step is lower bounded by the geometric random variable with parameter $\ell\frac{2^d}\beta$.
\begin{align*}
    \mathbb E \left[\textnormal{Geo}\left(\ell\frac{2^d}\beta\right)\right] = \frac\beta{2^d}\cdot\frac1\ell.
\end{align*}

It follows that the covering time can be bounded below by
\begin{align*}
    \mathbb E \left[\tau(T_\epsilon)\right] \geq \frac\beta{2^d} \sum_{\ell=1}^{\overline \beta} \frac1\ell.
\end{align*}
This gives a bound using the $\beta$-harmonic number
\begin{align*}
    \mathbb E \left[\tau(T_\epsilon)\right] &\geq \frac\beta{2^d} H_{\overline \beta} \sim \frac\beta{2^d} \log \beta = \Omega\left(\frac1{\epsilon^d} \log \left(\frac1\epsilon\right)\right).
\end{align*}
\end{proof}

For an upper bound we consider a grid over $[0,1]^d$ where each grid cell is a cube has side length $\frac\epsilon{\sqrt d}$. The time required for the \textsc{rrt} to have a vertex in each grid cell is greater than the covering time. To find upper bound we need the following lemma to bound the time required for a new cell to have a vertex.

\begin{lem}
Let $T_\epsilon$ be an \textsc{rrt} on $[0,1]^d$ with an $\frac\epsilon{\sqrt d}$ grid over it. Let $S_i$ be the union of the cells which do not contain a vertex of $T_\epsilon$ at time $i-1$. Then $X_i \in S_i$ if and only if $Y_i \in S_i$.
\end{lem}

\begin{proof}
Suppose that at some time $i,$ the cube $Y_i \in S_i$ and $X_i \notin S_i$. Clearly $X_i \not= Y_i$ so $d(X_{\gamma_i}, Y_i) = \epsilon + d(X_i,Y_i)$. But since $X_i \in S_i^c$, it belongs to a cell which contains $X_j$ for some $j < i$. It follows that

\begin{align*}
    d(X_j, Y_{i})& \leq d(X_j, X_{i}) + d(X_{i}, Y_{i}) \\
    &< \epsilon + d(X_i, Y_{i}) \\
    &= d(X_{\gamma_i}, Y_{i}).
\end{align*}

This contradicts the definition of $X_{\gamma_i}$ so we have proven the claim.
\end{proof}

\begin{figure}
\centering
\includegraphics[width=0.6\textwidth]{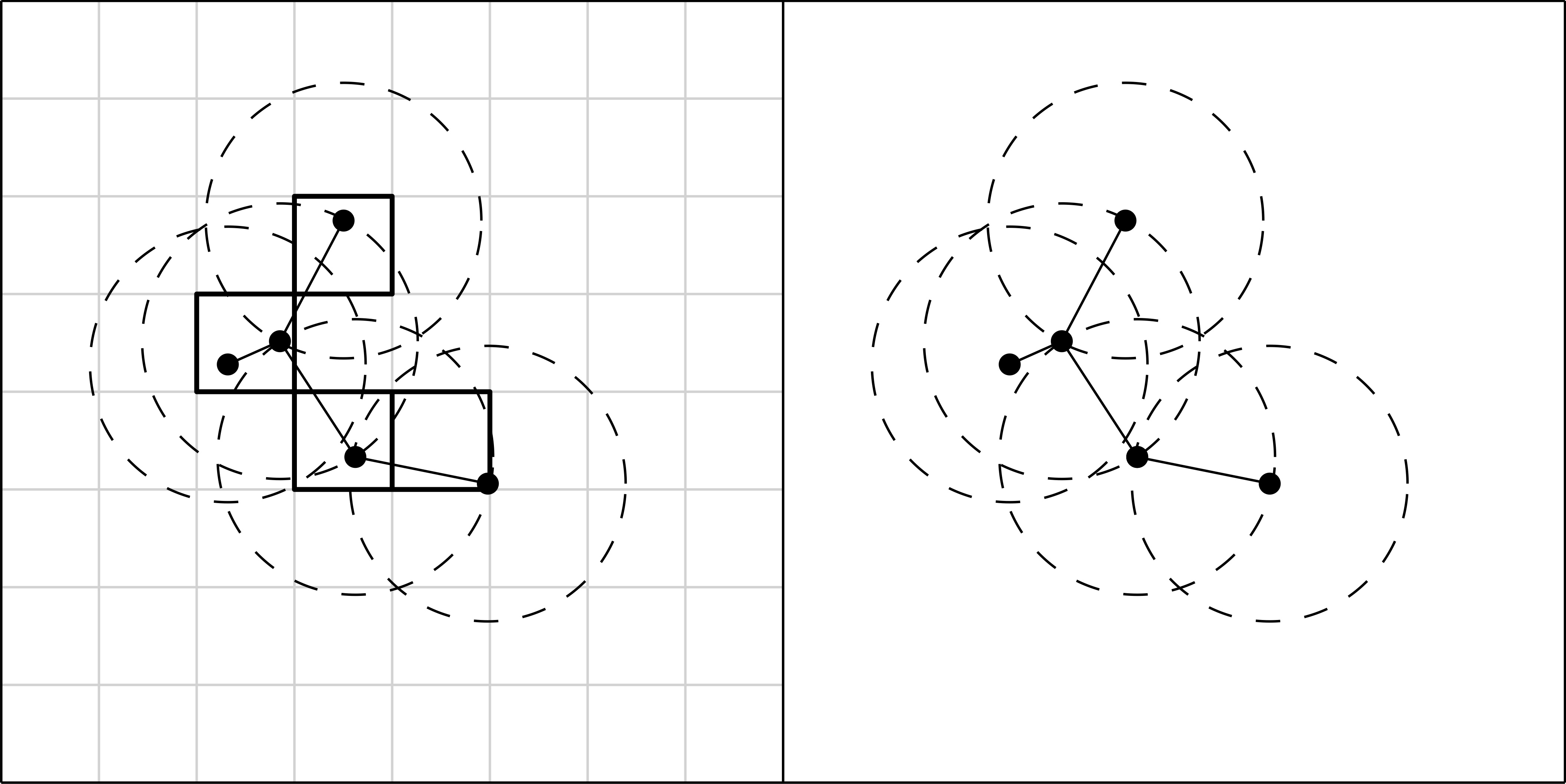}
\caption{On the right we see a section of the \textsc{rrt} and the $\epsilon$-balls that it covers. On the left the space is covered with a grid and the covered sections of the grid are bolded.}
\end{figure}

Now we continue to the upper bound.

\begin{proof}
For the $\frac\epsilon{\sqrt d}$ grid, each cell is covered by any vertex inside it. Thus the time required to have a vertex in each grid cell upper bounds $\mathbb E \left[\tau(T_\epsilon)\right]$. Define $\alpha$ to be the inverse of the volume of a cell:

\begin{align*}
    \alpha := \frac{d^\frac d2}{\epsilon^d}, \quad \overline \alpha = \left \lceil \alpha \right \rceil.
\end{align*}

Suppose $\ell$ grid cells contain a vertex. Then, by our previous lemma, the time to add a vertex to a new cell is the time to draw a uniform random point in a new cell. This is geometric with parameter
\begin{align*}
    \left(1- \frac\ell\alpha\right) = \left(\frac1\alpha \left(\alpha - \ell\right)\right)
\end{align*}
which has expectation
\begin{align*}
    \mathbb E \left[\textnormal{Geo}\left(\frac1\alpha \left(\alpha - \ell\right)\right)\right] = \frac\alpha{\alpha - \ell}.
\end{align*}
It follows that
\begin{align*}
    \mathbb E \left[\tau(T_\epsilon)\right] \leq \alpha \sum_{\ell=1}^{\overline \alpha} \frac1{\alpha - \ell}.
\end{align*}

Again, the $\alpha$-harmonic number gives an asymptotic bound
\begin{align*}
    \mathbb E \left[\tau(T_\epsilon)\right] &\leq \alpha H_{\overline \alpha} \sim \alpha \log \alpha = O\left(\frac1{\epsilon^d} \log \left(\frac1\epsilon\right)\right).
\end{align*}
Thus, we have both inequalities.
\end{proof}

Notice that this result also holds for modified \textsc{rrt} algorithms which maintain the property that the point drawn in an uncovered cube will cover some new cube.

As well, note that the lower and upper bounds given by this proof are asymptotically
\begin{align*}
	\frac\beta{2^d} \log \beta &= \frac{d \cdot \Gamma\left(\frac d2 + 1\right)}{2^d \pi^\frac d2 \epsilon^d} \log \left(\frac1\epsilon\right) \quad + \frac{\Gamma\left(\frac d2 + 1\right)}{2^d \pi^\frac d2 \epsilon^d} \log \left(\frac{\Gamma\left(\frac d2 + 1\right)}{\pi^\frac d2}\right),
\end{align*}
and
\begin{align*}
	\alpha \log \alpha = \frac{d^{\frac d2 + 1}}{\epsilon^d} \log \left(\frac1\epsilon\right) + \frac{d^{\frac d2 + 1}}{2\epsilon^d} \log d,
\end{align*}
respectively.

Consider the lower bound. If we fix $\epsilon < \frac1{2\pi}$ and then let $d \to \infty$, the run time is $e^{\Omega(d \log d)}$, confirming that the run time is at least exponential in the dimension of the space \cite{dalibard2009control}.

\section{Coupon Collector Problem}
It is worth noticing that our upper bound argument is very similar to a classical problem. We recall the Coupon Collector Problem \cite{motwani1995randomized, flajolet1992}:

\begin{thm}
Suppose there exist $n$ different types of coupons and each time a coupon is drawn it is drawn independently and uniformly at random from the $n$ classes, i.e., for each coupon drawn and for each $i \in 1,..,n$ the probability that the coupon is of the $i$-th type is $\frac1n$. The expected time $T$ to have one of each type of coupon is asymptotically $n \log n$.
\end{thm}

In our upper bound argument for the covering time of the graph we have a scenario very similar to the Coupon Collector Problem. We have divided the space into a number of cubes that we need to cover, however in our case the selection of a new cube does not guarantee that it is covered, just that some new cube is covered. In this way we see it does not matter that we are not adding the randomly selected point to our tree.

\section{A Connecting Lemma---Nearest Neighbour Trees}
We turn our attention to the growth of the tree after the space has been covered. From this time onward, every point of $M$ is within distance $\epsilon$ of the tree, so we will have $X_i = Y_i$ after coverage. Our analysis between the growth of the \textsc{rrt} in this phase depends on the relationship between the \textsc{rrt} and the Nearest Neighbour Tree (\textsc{nnt}).

\begin{dfn}
An \textsc{nnt} is a sequence of vertex and edge random variable pairs
\begin{align*}
    T_{\textsc{nn}} = T_{\textsc{nn}}(M, v) = ((X_i,E_i))_{i \in \mathbb N}
\end{align*}
defined on a compact metric space $M$ with an initial vertex $v = X_0$ and $E_0 = \emptyset$. Each $(X_i,E_i), i \geq 1$ is generated in the following way: take a sequence of uniform $M$-valued random variables $(X_i)_{j=1}^\infty$ and let 
\begin{align*}
    \gamma_i = \textnormal{argmin}_{0 \leq j \leq i-1} d(X_i,X_j).
\end{align*}
Then $X_i$ is added to the tree via the edge
\begin{gather}
    E_i := (X_{\gamma_i},X_i).
\end{gather}
An \textsc{nnt} at time $n \in \mathbb N$ is the subset $T_\textsc{nn}^n(M,v)$ where
\begin{align*}
    T_\textsc{nn}^n = T_\textsc{nn}^n(M, v) = ((X_i,E_i))_{0 \leq i \leq n}.
\end{align*}
\end{dfn}

Now we define a nearest neighbour process which grows onto a given tree $\Lambda$---note that $\Lambda$ may be a random tree itself. The purpose of this is to ultimately view an \textsc{rrt} $T_\epsilon$ as the growth of a nearest neighbour tree onto $T_\epsilon^{\tau(T_\epsilon)},$ the \textsc{rrt} at covering time.

\begin{figure}
\centering
\includegraphics[width=0.8\textwidth]{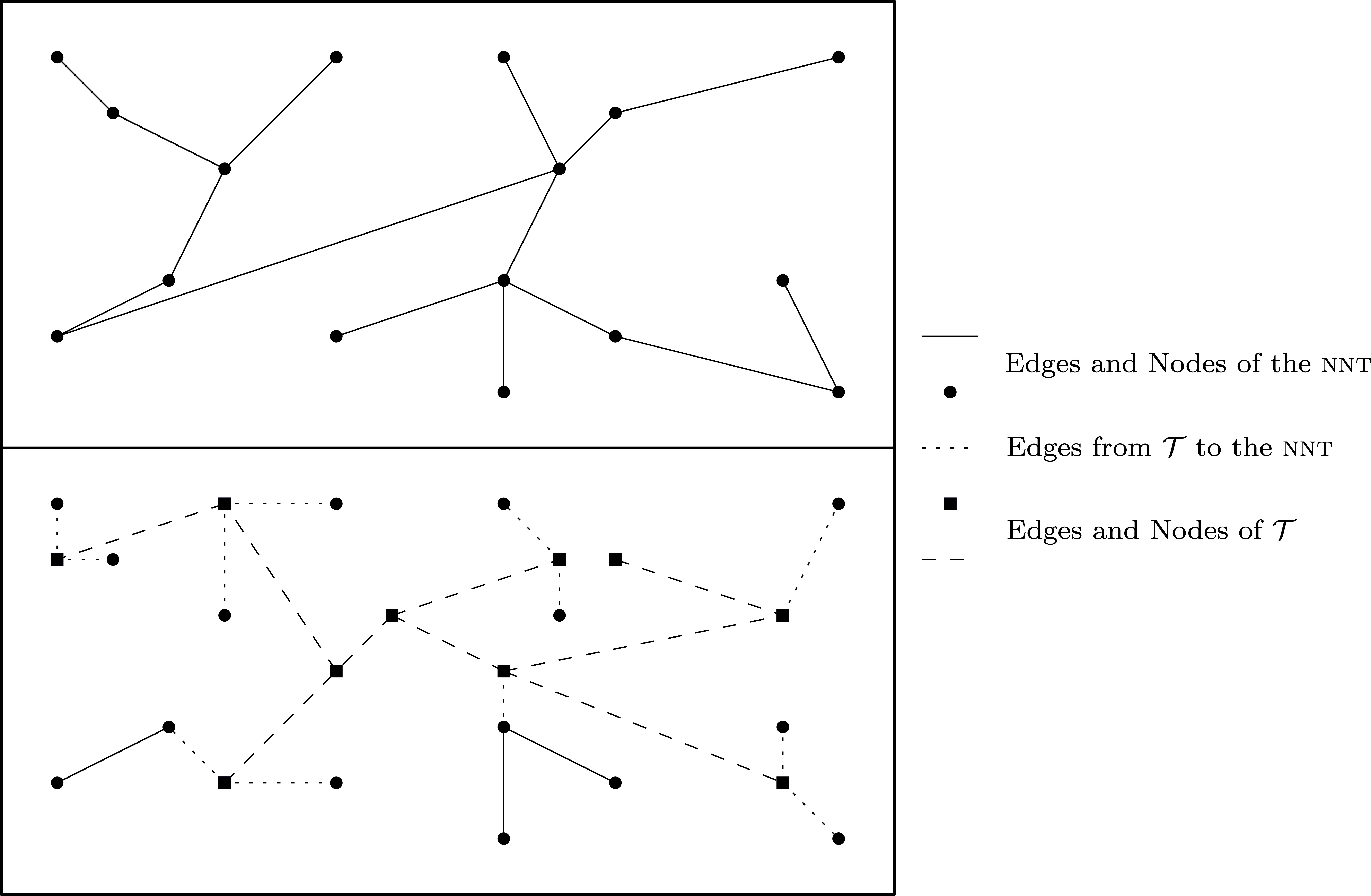}
\caption{On top is an \textsc{nnt}. On the bottom is the \textsc{nnt} grown onto another tree $\Lambda$.}
\end{figure}

\begin{dfn}
Let $S$ be a random variable in $\mathbb N \cup \{\infty\}$ and let $\Lambda$ be a random tree embedded in a metric space $M$ with vertices $(V_i \in M, i \in \{0,...,S\})$ and edges $(F_i, i \in \{1,...,S\})$. Let $T_{\textsc{nn}} = ((Z_i,E_i))_{i \in \mathbb N}$ be an \textsc{nnt}. The connection process of an \textsc{nnt} onto $\Lambda$ is the sequence of vertex and edge random variable pairs
\begin{align*}
	T_\Lambda = ((X_i,E_i'))_{i \in \mathbb N},
\end{align*}
where
\begin{align*}
    X_i &:= 
    \begin{cases}
      V_i &: \textnormal{if}\ 0 \leq i \leq S, \\
      Z_{i-M} &: \textnormal{otherwise} ,
    \end{cases} \\
    \gamma_i &= \textnormal{argmin}_{0 \leq j \leq i-1} d(X_i,X_j), \\
\end{align*}
and
\begin{align*}
    E_i' = 
    \begin{cases}
      \emptyset &: \textnormal{if}\ i = 0 \\
      F_i &: \textnormal{if}\ 1 \leq i \leq S, \\
      (X_i,X_{\gamma_i}) &: \textnormal{otherwise}.
    \end{cases}
\end{align*}
\end{dfn}

It is worth noting that the process loses a lot of meaning if $\mathbb P \left[S = \infty\right] \not= 0$. As well, $\Lambda$ may be deterministic in which case its associated random variables are all constant.

The Connecting Lemma below links the \textsc{nnt} to the \textsc{rrt}.

\begin{lem}
Let $T_\Lambda = ((X_i,E_i'))_{i \in \mathbb N}$ be the process defined above. $T_\Lambda$ has the following properties:
\begin{itemize}
    \item Let $\delta_n$ be the distance from $Z_n$ to its parent in $T_\textsc{nn}$ and $\delta_n'$ be the distance from $X_{n + S}$ to its parent in $T_\Lambda$. Then $\delta_n \geq \delta_n'$.
    \item Let $H(\Lambda)$ be the height of $\Lambda$, $D_n$ be the depth of $Z_n$ in $T_\textsc{nn}$ and $D_{m}'$ be the depth of $X_m$ in $T_\Lambda$. Then for $m = S + n$,
    \begin{align*}
        D_m' \leq D_n + H(\Lambda) + 1.
    \end{align*}
\end{itemize}
\end{lem}

\begin{proof}
The first point follows by definition. For the second, let $m = n + M$ and let
\begin{align*}
    \alpha = \textnormal{argmin}_{k < i \leq m} \{X_i : X_i \textnormal{ is an ancestor of } X_m\}.
\end{align*}
The length of the path from $X_m$ to $X_\alpha$ is certainly less than $D_n$ and the path up the tree starting at $X_\alpha$'s parent is contained in $\Lambda$, so it follows that
\begin{align*}
    D_m' &\leq D_n + D_\alpha' \leq D_n + H(\Lambda) + 1.
\end{align*}
\end{proof}

Put simply, the first property of the Connecting Lemma states that the distance from a node to its parent is smaller in a connection process than an \textsc{nnt}. The second property states that the depth of a node in a connection process is less than the depth of the node in the \textsc{nnt} and an $O(H(\Lambda))$ term---importantly, in the case of the \textsc{rrt} $H(\Lambda) = O(1)$. Armed with this lemma, we can now analyze the behaviour of the \textsc{rrt} after covering.

\section{Phase 3---Growth After Covering}

\begin{figure}
\centering
\includegraphics[width=0.7\textwidth]{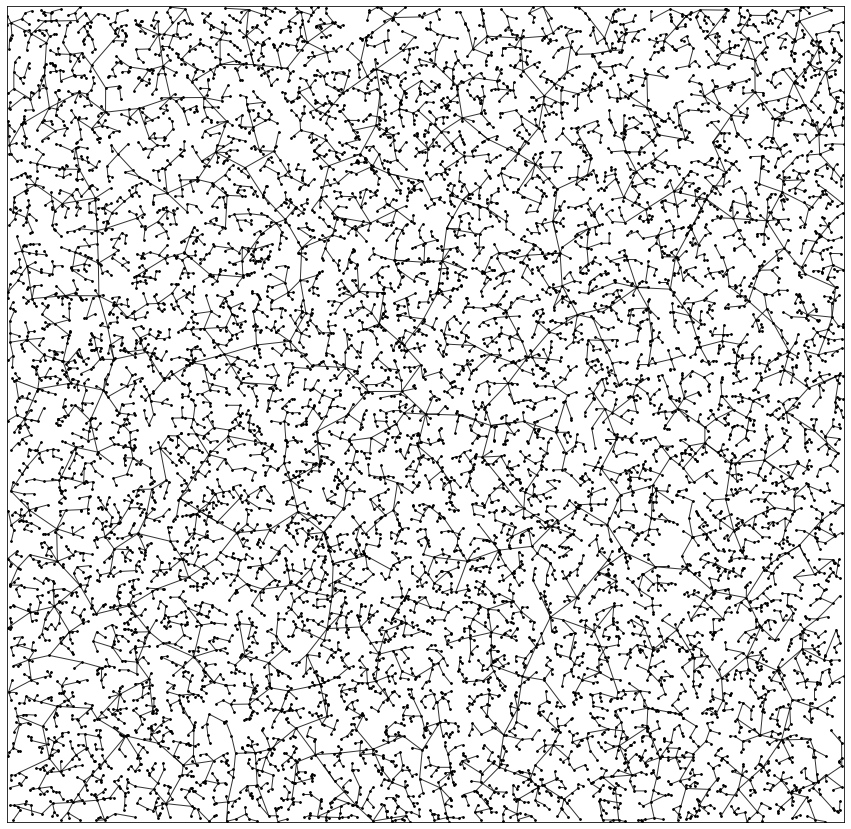}
\caption{An \textsc{rrt} well after covering time. The tree is growing denser in the space and the whole space is covered.}
\end{figure}

We first analyze the behaviour of the \textsc{nnt} and then relate it to the \textsc{rrt} via the connecting lemma. In the case of an \textsc{rrt}, $\Lambda = T_\epsilon^{\tau(T_\epsilon)}$ and $T_\textsc{nn} = T_\epsilon - \Lambda$.

\subsection{Euclidean Distances}

Let $\textnormal{Exp}(\lambda)$ be a random variable $X$ with distribution
\begin{align*}
    \mathbb P \left[X \leq x\right] = 1 - e^{-\lambda x}.
\end{align*}
We begin with a Lemma on the \textsc{nnt}.

\begin{lem}
Let $T_\textsc{nn}([0,1]^2) = ((X_i,E_i))_{i \in \mathbb N}$ be an \textsc{nnt}. Then $\delta_n$ obeys the following inequality
\begin{align*}
    \limsup_{n \to \infty} \mathbb E \left[\delta_n \sqrt{\pi n}\right] &\leq \mathbb E \left[\sqrt {\textnormal{Exp}(1)}\right] = \frac{\sqrt\pi}2.
\end{align*}
\end{lem}

\begin{proof}
If $\delta_n > x$, then $X_1,...,X_{n-1}$ all lie outside the ball $B(X_n,x) \cap [0,1]^2$. It follows by independence that
\begin{align*}
    \mathbb P \left[\delta_n > x\right] \leq \left(1 - \pi x^2\right)^{n-1}.
\end{align*}
Considering $\delta_n \sqrt{\pi n}$ then gives
\begin{align*}
    \mathbb P \left[\delta_n \sqrt{\pi n} > x\right] \leq \left(1 - \frac{x^2}n\right)^{n-1}.
\end{align*}
Taking the limit, we get
\begin{align*}
    \mathbb P \left[\delta_n \sqrt{\pi n} \leq x\right] &\leq 1 - \left(1 - \frac{x^2}n\right)^{n-1} \ensuremath{\stackrel{n \to \infty}{\longrightarrow}} 1 - e^{-x^2}.
\end{align*}
This is the distribution of $\sqrt{\textnormal{Exp}(1)}$, so
\begin{align*}
    \limsup_{n \to \infty} \mathbb E \left[\delta_n \sqrt{\pi n}\right] &\leq \mathbb E \left[\sqrt{\textnormal{Exp}(1)}\right] = \frac{\sqrt\pi}2 .
\end{align*}
This gives our result.
\end{proof}

Now we use this lemma for a result on the \textsc{rrt}.

\begin{thm}
Let $T_\epsilon([0,1]^2) = ((X_i,E_i))_{i \in \mathbb N}$ be an \textsc{rrt}. Then $\delta_m'$ obeys the following inequality
\begin{align*}
    \limsup_{m \to \infty} \mathbb E \left[\delta_m' \sqrt{\pi m}\right] \leq \frac{\sqrt\pi}2.
\end{align*}
\end{thm}

\begin{proof}
Consider $\tau(T_\epsilon)$, the covering time of $T_\epsilon$. Recall that $\tau(T_\epsilon)$ is finite with probability 1. Define $N_m := (m - \tau(T_\epsilon)) \mathbf 1_{\left[m \geq \tau(T_\epsilon)\right]}.$ Then we have
\begin{align*}
    m = \mathbf 1_{\left[m \geq \tau(T_\epsilon)\right]} \left(\tau(T_\epsilon) + N_m\right) + m \mathbf 1_{\left[m < \tau(T_\epsilon)\right]}.
\end{align*}
Now, by the Connecting Lemma,
\begin{align*}
    \mathbb E \left[\delta_m' \sqrt{\pi m}\right] &\leq \mathbb E \left[\delta_{N_m + \tau(T_\epsilon)}' \mathbf 1_{\left[m \geq \tau(T_\epsilon)\right]} \sqrt{\pi (N_m + \tau(T_\epsilon))}\right] \\
    &\quad + \mathbb E \left[\delta_m' \mathbf 1_{\left[m < \tau(T_\epsilon)\right]} \sqrt{\pi m}\right].
\end{align*}
For the second term, by Cauchy-Schwarz,
\begin{align*}
	&\mathbb E \left[\delta_m' \mathbf 1_{\left[m < \tau(T_\epsilon)\right]} \sqrt{\pi m}\right]
	\leq \mathbb P \left[m < \tau(T_\epsilon)\right]^\frac12 \mathbb E \left[(\delta_m')^2 \pi m\right]^{\frac12},
\end{align*}
and since
\begin{align*}
	\limsup_{m \to \infty} \mathbb E\left[(\delta_m')^2 \pi m\right] \leq \mathbb E \left[\textnormal{Exp}(1)\right] = 1,
\end{align*}
it follows that
\begin{align*}
	&\limsup_{m \to \infty} \mathbb E \left[\delta_m' \mathbf 1_{\left[m < \tau(T_\epsilon)\right]} \sqrt{\pi m}\right] \leq \limsup_{m \to \infty} \mathbb P \left[m < \tau(T_\epsilon)\right]^\frac12 = 0.
\end{align*}

Next we split the first term. By the Connecting Lemma and Cauchy-Schwarz
\begin{align*}
    &\mathbb E \left[\delta_{N_m + \tau(T_\epsilon)}' \mathbf 1_{\left[m \geq \tau(T_\epsilon)\right]} \sqrt{\pi (N_m + \tau(T_\epsilon)}\right] \\
     \leq\ &\mathbb E \left[\delta_{N_m} \mathbf 1_{\left[m \geq \tau(T_\epsilon)\right]} \sqrt{\pi N_m}\right] + \mathbb E \left[\delta_m^2\right]^\frac12 \mathbb E \left[\tau(T_\epsilon)^2\right]^\frac12.
\end{align*}
In the second term above, recall that $\tau(T_\epsilon)$ can be upper bounded by a finite sum of independent random variables. Clearly $\mathbb E \left[\delta_m^2\right] = o(1)$, so it follows that
\begin{align*}
    \mathbb E \left[\delta_m' \sqrt{\pi m}\right] \leq \mathbb E \left[\delta_{N_m} \mathbf 1_{\left[m \geq \tau(T_\epsilon)\right]} \sqrt{\pi N_m}\right] + o(1).
\end{align*}
Now we split the first term again, this time for values of $N_m$ around $m - \sqrt m$. For $m > 4$
\begin{align*}
	&\delta_{N_m} \mathbf 1_{\left[m \geq \tau(T_\epsilon)\right]} \sqrt{\pi N_m} \\
	\leq\ &\mathbf 1_{\left[N_m \leq \left(m -\sqrt m\right)\right]} \left(m - \sqrt m\right) + \mathbf 1_{\left[N_m \geq \left(m - \sqrt m\right)\right]} \delta_{N_m} \sqrt{\pi m} \\
	=\ &\mathbf 1_{\left[\tau(T_\epsilon) \geq \sqrt m\right]} \left(m - \sqrt m\right) + \mathbf 1_{\left[N_m \geq \left(m - \sqrt m\right)\right]} \delta_{N_m} \sqrt{\pi m}.
\end{align*}
Let $\alpha = \left\lfloor \frac1{1 - \pi\epsilon^2} \right\rfloor$. At any time before covering the probability of failure is bounded above by $\frac1\alpha$ while success we can bound trivially by 1. As well, in order to have $\tau(T_\epsilon) \geq \sqrt m$ we must have at least $\sqrt m - \alpha-1$ failures in the first $\sqrt m$ steps, so we have the following bound on the first term above:
\begin{align*}
	\mathbb E \left[\mathbf 1_{\left[\tau(T_\epsilon) \geq \sqrt m\right]} \left(m - \sqrt m \right)\right] &\leq \left(m - \sqrt m\right) \mathbb P \left[\tau(T_\epsilon) \geq \sqrt m\right] \\
	&\leq \left(m - \sqrt m\right) \binom{\sqrt m}{\alpha} \alpha^{-\sqrt m} \alpha^{\alpha + 1} \\
	& = o(1).
\end{align*}
Finally, we find a bound on the most significant term. Note that for $\ell > 0$
\begin{align*}
	\mathbb E \left[\mathbf 1_{\left[N_m \geq \left(m - \sqrt m\right)\right]} \delta_{N_m} \sqrt{\pi m}\right] \leq \sqrt{\frac m{m - \sqrt m}} \mathbb E \left[\delta_{m - \sqrt m}\sqrt{\pi \left(m - \sqrt m\right)}\right].
\end{align*}
Combining the above inequalities gives gives us that
\begin{align*}
	\mathbb E \left[\delta_m' \sqrt{\pi m}\right] \leq \sqrt{\frac m{m - \sqrt m}} \mathbb E \left[\delta_{m - \sqrt m}\sqrt{\pi \left(m - \sqrt m\right)}\right] + o(1).
\end{align*}
Taking the limit we get our desired result
\begin{align*}
    \limsup_{m \to \infty} \mathbb E \left[\delta_m' \sqrt{\pi m}\right] \leq \limsup_{m \to \infty} \sqrt{\frac {m + \sqrt m}m} \mathbb E \left[\delta_{m}\sqrt{\pi m}\right] \leq \frac{\sqrt\pi}2.
\end{align*}
\end{proof}

This quickly gives a result on the tree that is more tangible.

\begin{cor}
Let $T_\epsilon([0,1]^2) = ((X_i,E_i))_{i \in \mathbb N}$ be an \textsc{rrt} and let $\Delta_n = \sum_{i=1}^n \delta_i'$, the total Euclidean path length at time $n$. Then
\begin{align*}
    \mathbb E \left[\Delta_n\right] = O\left(\sqrt {n}\right).
\end{align*}
\end{cor}

\begin{proof}
By Theorem 4, we know that
\begin{align*}
    \mathbb E \left[\delta_n'\right] = O\left(\frac1{\sqrt{n}}\right).
\end{align*}

It follows that
\begin{align*}
    \mathbb E \left[\Delta_n\right] &= \sum_{i=1}^n \mathbb E \left[\delta_i'\right] = O\left(\sum_{i=1}^n \left(\frac1{\sqrt{i}}\right)\right) = O\left(\sqrt n\right).
\end{align*}
\end{proof}

Our final result on Euclidean distances concerns the path from a single node to the root.

\begin{thm}
Let $T_\epsilon([0,1]^2) = ((X_i,E_i))_{i \in \mathbb N}$ be an \textsc{rrt} and let $L_n$ be the Euclidean path length from $X_n$ to $X_0$. Then
\begin{align*}
    \mathbb E \left[L_n\right] = O(1).
\end{align*}
\end{thm}

\begin{proof}
First, we observe a useful equality where $A_i$ is the event that $X_i$ is an ancestor of $X_n$ and $\mathbf 1_{[A]}$ is the indicator function:
\begin{align*}
    L_n \leq \sum_{i=1}^{n-1} \mathbf 1_{\left[A_i\right]} \delta_i + \delta_n + \epsilon \cdot \tau(T_\epsilon).
\end{align*}
Now cover $[0,1]^2$ with a grid of side length $\frac{\log i}{\sqrt i}$. Letting $S_i$ be the grid square containing $X_i$ and letting $E_i$ be the event that $\{X_0,...,X_{i-1}\} \cap S_i = \emptyset$, we get the inequality
\begin{align*}
    \delta_i \leq \mathbf 1_{\left[E_i\right]} \frac{\sqrt{2}\log i}{\sqrt i} + \sqrt2 \mathbf 1_{\left[E_i^c\right]}.
\end{align*}
Next, sort the nodes $\{X_0,...,X_{n-1}\}$ by distance to $X_n$. Then by this sorting $A_i$ is the event that $X_i$ is a record, so $\mathbb P \left[A_i\right] = \frac1i$ \cite{GlickRecords}. Furthermore, $A_i$ is independent of the event $E_i$. It follows that
\begin{align*}
    \mathbb E \left[L_n\right] &\leq \mathbb E \left[\sum_{i=1}^{n-1} \mathbf 1_{\left[A_i\right]} \cdot \delta_i + \delta_n\right] + \mathbb E \left[\tau(T_\epsilon)\right] \\
    &\leq \sqrt2 + \mathbb E \left[\tau(T_\epsilon)\right] + \sum_{i=1}^{n-1} \mathbb E \left[\frac{\sqrt{2}\log i}{\sqrt i} \cdot \mathbf 1_{\left[A_i\right]} \mathbf 1_{\left[E_i\right]} + \sqrt 2 \cdot \mathbf 1_{\left[A_i\right]} \mathbf 1_{\left[E_i^c\right]}\right] \\
    &\leq \sqrt2 + \mathbb E \left[\tau(T_\epsilon)\right] +  \sqrt2 \sum_{i=1}^{n-1} \frac{\log i}{i^\frac32} + \sum_{i=1}^{n-1} \left(1 - \frac{(\log i)^2}i\right)^{i-1} \\
    &= O(1).
\end{align*}
\end{proof}

\subsection{Height and Depth}
Next, we consider the height of the tree and depth of the $n$-th node. We note the similarity with the analysis of the height of a binary search tree \cite{devroye1986note}.

To begin, we analyze the height of the \textsc{nnt}.

\begin{lem}
Let $T_\textsc{nn}([0,1]^2) = ((X_i,E_i))_{i \in \mathbb N}$ be an \textsc{nnt}, let $D_n$ be the depth of $X_n$, and let $H_n$ be the height of $T_\textsc{nn}^n$. For every $\epsilon > 0$, we have
\begin{align*}
    &\lim_{n \to \infty} \mathbb P \left[D_n > (1 + \epsilon) \log n\right] = 0, \\
    &\lim_{n \to \infty} \mathbb P \left[H_n > (1 + \epsilon) e \log n\right] = 0.
\end{align*}
\end{lem}

\begin{proof}
Observe that the index of $X_n$'s parent is uniformly distributed on $\{0,...,n-1\}$ due to independence. Stated another way, it is distributed as $\left\lfloor n U \right \rfloor$ where $U$ is uniform on $[0,1].$ Iterating this process will take us to $X_n$'s parent, then grandparent, then eventually the root. It follows that
\begin{align*}
    D_n = \inf \left\{k \in \mathbb N : \left \lfloor \left \lfloor \left \lfloor n U_1 \right \rfloor U_2 \right \rfloor ...U_k \right \rfloor = 0 \right\}.
\end{align*}
where the $U_i$'s are independent. For a useful bound on the $D_n$, we upper bound by removing the floors and using Markov's inequality:
\begin{align*}
    \mathbb P \left[D_n > x\right] &= \mathbb P \left[ \left \lfloor \left \lfloor \left \lfloor n U_1 \right \rfloor U_2 \right \rfloor ...U_x \right \rfloor > 0 \right] \\
    &\leq \mathbb P \left[n \prod_{i=1}^x U_i \geq 1\right] \\
    &\leq \inf_{\lambda > 0} \mathbb E \left[n^\lambda \prod_{i=1}^x U_i^\lambda \right] \\
    &= \inf_{\lambda > 0} n^\lambda \frac1{(1+\lambda)^x} \\
    &= \inf_{\lambda > 0} e^{\lambda \log n} \frac1{(1+\lambda)^x}.
\end{align*}
When $x > \log n$, the infimum is achieved at
\begin{align*}
    1 + \lambda = \frac{x}{\log n}.
\end{align*}

It follows that
\begin{align*}
    \mathbb P \left[D_n > x\right] &\leq e^{x - \log n} \left(\frac{\log n}{x}\right)^x = \frac1n \left(\frac{e \log n}{x}\right)^x.
\end{align*}

If $x = (1+\epsilon)\log n$, then with
\begin{align*}
    \phi(\epsilon) = \epsilon - (1-\epsilon) \log (1+\epsilon),
\end{align*}
we have
\begin{align*}
    \mathbb P \left[D_n > (1+\epsilon)\log n\right] &\leq \frac1n \left(\frac{e}{(1+\epsilon)}\right)^{(1+\epsilon)\log n} = n^{\phi(\epsilon)}.
\end{align*}

Note that for $\epsilon > 0$, $\phi(\epsilon) < 0$. It follows that
\begin{align*}
    \mathbb P \left[D_n > (1+\epsilon)\log n\right] \ensuremath{\stackrel{n \to \infty}{\longrightarrow}} 0.
\end{align*}

Next, we can bound the height of the tree by
\begin{align*}
    \mathbb P \left[H_n > x\right] &\leq \sum_{i=1}^n \mathbb P \left[D_i > x\right] \leq n \mathbb P \left[D_n > x\right] \leq \left(\frac{e \log n}{x}\right)^{x}.
\end{align*}
If $x = (1+\epsilon)e\log n$, then
\begin{align*}
    &\mathbb P \left[H_n > (1+\epsilon)e\log n\right] \leq \left(\frac1{1+\epsilon}\right)^{(1+\epsilon)e\log n} \ensuremath{\stackrel{n \to \infty}{\longrightarrow}} 0.
\end{align*}
\end{proof}

The Connecting Lemma extends this result to the \textsc{rrt}.

\begin{thm}
Let $T_\epsilon([0,1]^2) = ((X_i,E_i))_{i \in \mathbb N}$ be an \textsc{rrt}, let $D_n'$ be the depth of $X_n$, and let $H_n'$ be the height of $T_\epsilon^n$. For every $\delta > 0$, we have
\begin{align*}
    &\lim_{n \to \infty} \mathbb P \left[D_n' > (1 + \delta) \log n\right] = 0, \\
    &\lim_{n \to \infty} \mathbb P \left[H_n' > (1 + \delta) e \log n\right] = 0.
\end{align*}
\end{thm}

\begin{proof}
Simply observe that by the Connecting Lemma
\begin{align*}
    \{D_n' > (1 + \delta) \log n\}	\subset\ &\left\{D_n > \left(1 + \frac\delta2\right) \log n\right\} \cup \left\{\tau(T_\epsilon) > \frac \delta 2 \log n \right\}.
\end{align*}

It follows that
\begin{align*}
    &\lim_{n \to \infty} \mathbb P \left[D_n' > (1 + \delta) \log n\right] \\
    \leq\ &\lim_{n \to \infty} \mathbb P \left[D_n > \left(1 + \frac\delta2\right) \log n\right] + \lim_{n \to \infty} \mathbb P \left[\tau(T_\epsilon) > \frac \delta 2 \log n\right] \\
    =\ &0.
\end{align*}
The result for $H_n'$ follows similarly.
\end{proof}

We give a bound on the expectation as well. Again, we start with a lemma on the \textsc{nnt}.

\begin{lem}
Let $T_\textsc{nn}([0,1]^2) = ((X_i,E_i))_{i \in \mathbb N}$ be an \textsc{nnt}, let $D_n$ be the depth of $X_n$, and let $H_n$ be the height of $T_\textsc{nn}^n$. Then
\begin{align*}
    \mathbb E \left[D_n\right] &\leq (1 + o(1)) \log n, \\
    \mathbb E \left[H_n\right] &\leq (e + o(1)) \log n.
\end{align*}
\end{lem}

\begin{proof}
Note that for $\epsilon > 0$ fixed,
\begin{align*}
    \mathbb E \left[D_n\right] &= \sum_{i = 0}^\infty \mathbb P\left[D_n > i\right] \\
    &\leq (1 + \epsilon) \log n \\
    &\quad + \sum_{i = \left \lceil (1 + \epsilon)\log n \right \rceil}^{\left \lfloor 2e\log n \right \rfloor} \mathbb P \left[D_n > i\right] \\
    &\quad + \sum_{i = \left \lfloor 2e\log n \right \rfloor + 1}^n \mathbb P\left[D_n > i\right].
\end{align*}

For the third term, note that due to our bound in Lemma 4, for $x \geq 2e\log n$,
\begin{align*}
    \sum_{i = \left \lfloor 2e\log n \right \rfloor + 1}^n \mathbb P \left[D_n > i\right] &\leq \sum_{i = \left \lfloor 2e\log n \right \rfloor + 1}^n \mathbb P \left[H_n > i\right] \\
    &\leq \sum_{i = \left \lfloor 2e\log n \right \rfloor + 1}^\infty \frac1{2^i} \\
    &\leq \frac2{2^{2e \log n}} \\
    &= o(1).
\end{align*}

For the second term, we have the bound
\begin{align*}
    &\sum_{i = \left \lceil (1 + \epsilon)\log n \right \rceil}^{\left \lfloor 2e\log n \right \rfloor} \mathbb P \left[D_n > i\right] \leq 2e\log n \cdot \mathbb P \left[D_n > (1 + \epsilon) \log n\right].
\end{align*}

Again, by Lemma 4,
\begin{align*}
    &\sum_{i = \left \lceil (1 + \epsilon)\log n \right \rceil}^{\left \lfloor 2e\log n \right \rfloor} \mathbb P \left[D_n > i\right] \leq 2e\log n \cdot n^{\phi(\epsilon)} = o(1).
\end{align*}

Thus,
\begin{align*}
    \mathbb E \left[D_n\right] &\leq (1 + \epsilon) \log n + o(1) = (1 + o(1)) \log n.
\end{align*}

For $H_n$ and fixed $\epsilon > 0$ we have similar bounds:
\begin{align*}
    \mathbb E \left[H_n\right] &\leq (e + \epsilon) \log n \\
    &\quad + \sum_{i = \left\lceil(e + \epsilon) \log n \right \rceil}^{\left \lfloor 2e\log n \right \rfloor} \mathbb P\left[H_n > i\right] \\
    &\quad + \sum_{i = \left \lfloor 2e\log n \right \rfloor + 1}^n \mathbb P \left[H_n > i\right].
\end{align*}

We have already shown the third term is $o(1)$, so we turn our attention to the second term. We again use a bound from Lemma 4 to get
\begin{align*}
     \sum_{i = \left\lceil(e + \epsilon) \log n \right \rceil}^{\left \lfloor 2e\log n \right \rfloor} \mathbb P \left[H_n > i\right] &\leq 2e \log n \cdot \left(\frac e{e + \epsilon}\right)^{(e + \epsilon)\log n} \\
    &= 2e \log n \cdot n^{(e + \epsilon)(1 - \log(e + \epsilon))} \\
    &= o(1).
\end{align*}
Thus,
\begin{align*}
    \mathbb E \left[H_n\right] \leq (e + o(1)) \log n.
\end{align*}
\end{proof}

The Connecting Lemma also allows an easy extension of this result to the \textsc{rrt}.

\begin{thm}
Let $T_\epsilon([0,1]^2) = ((X_i,E_i))_{i \in \mathbb N}$ be an \textsc{rrt}, let $D_n'$ be the depth of $X_n$, and let $H_n'$ be the height of $T_\epsilon^n$. Then
\begin{align*}
    \mathbb E \left[D_n'\right] &\leq (1 + o(1)) \log n, \\
    \mathbb E \left[H_n'\right] &\leq (e + o(1)) \log n.
\end{align*}
\end{thm}

\begin{proof}
By the Connecting Lemma
\begin{align*}
	\mathbb E \left[D_n'\right] &\leq \mathbb E\left[D_n\right] + \mathbb E \left[\tau(T_\epsilon)\right] \leq (1 + o(1)) \log n + O(1) \sim \log n.
\end{align*}
The result for $H_n'$ follows similarly.
\end{proof}

\section{Future Work}
There is a lot of work to be done using this framework. Immediate questions we are working on are:
\begin{itemize}
    \item Is the bound on reaching a convex region of positive probability tight? As the question is open, there remains hope that the expected time is in fact its trivial lower bound, $\Theta\left(\frac1\epsilon\right)$. Further, how long does it take the tree to reach the boundary of the space?
    \item What is the expected run time of \textsc{rrt}-Connect? Finding a path should be much faster than finding a single point, and growing two trees into each other should greatly improve the performance.
    \item How does the performance change when barriers are added? For a specific space it is relatively easy to use this framework to determine an expected run time, but what can be said about a general space?
\end{itemize}

Our approach for the question of barriers will be to study random barriers generated by percolation models and determine expected run times on these spaces. Specifically, we will study site percolation as presented by Broadbent and Hammersly \cite{broadbent1957percolation} and lilypad percolation as presented by Gilbert \cite{gilbert1961random}.

For site percolation, we take the space $[0,1]^d$ and drop a grid on it of side length $\epsilon$. For chosen probability $p$---which we will vary to model spaces with different quantities of barriers---we will place a barrier at each grid cube. We then run the \textsc{rrt} on the largest connected component of the space we have created and ask ourselves the same questions as we have asked for $[0,1]^d$: what is the expected covering time of the connected component? What is the expected time for the \textsc{rrt} to spread through the space? What is the expected run time of \textsc{rrt}-Connect?

For lilypad percolation we fill the space with $n$---again, we will vary $n$---discs whose radius is a function of $\epsilon$ and whose centers are chosen uniformly and independently in $[0,1]^d$.

\section{Conclusion}
Our results prove the conjectured performance of the \textsc{rrt}, answer some questions on its performance, and give a greater understanding to its behaviour while providing a framework for further analysis of the tree.

To summarize, we have split the growth of the \textsc{rrt} into three phases: initial growth, covering behaviour, and post covering growth. In the initial phase we have studied the time it takes the \textsc{rrt} to reach a convex area of positive probability, furthering the study of Arnold et al. \cite{arnold2013convex} on how fast the tree spreads through a space .

In the covering phase we have determined the run time the tree needs to grow close to every point in the space. We have found that with step size $\epsilon,$ the run time of the algorithm as $\epsilon \to 0$ is in expectation
\begin{align*}
    \Theta\left(\frac1{\epsilon^d} \log \left(\frac1\epsilon\right)\right),
\end{align*}
and that fixing $\epsilon,$ the expected run time is $e^{\Omega(d\log d)}$. This confirms the conjecture that the run time is worse than exponential in the dimension and provides a good understanding of the performance as a function of step size.

In the post-covering phase we have used a relationship the the Nearest Neighbour Tree to examine its behaviour. As a result, we know that asymptotically the path length and total Euclidean path length are both $O(\sqrt n)$. As well, the expected depth of the $n$-th node and height of the tree at time $n$ are both $O(\log n)$.

\printbibliography 

\end{document}